%% file: causal-bounds-main.tex
\documentclass{article}

\PassOptionsToPackage{numbers, compress}{natbib}

\usepackage[preprint]{neurips_2019}

\usepackage[utf8]{inputenc} 
\usepackage[T1]{fontenc}    
\usepackage{hyperref}       
\usepackage{url}            
\usepackage{booktabs}       
\usepackage{amsfonts}       
\usepackage{nicefrac}       
\usepackage{microtype}      
\usepackage{graphicx}
\usepackage{graphics}
\usepackage{subfig}
\usepackage{amsmath}
\usepackage{amsthm}
\usepackage{amssymb}
\usepackage{xcolor}
\usepackage{enumitem}
\usepackage{tabu}
\usepackage[compact]{titlesec}    
\titlespacing{\section}{0pt}{*0}{*0}   
\titlespacing{\subsection}{0pt}{*0}{*0}
\titlespacing{\subsubsection}{0pt}{*0}{*0}   
\title{Quantifying Error in the Presence of Confounders for Causal Inference}
\author{Rathin Desai \\
Microsoft Research India \\
Bangalore, Karnataka \\
\texttt{desairathin18@gmail.com}
\And
Amit Sharma \\
Microsoft Research India \\
Bangalore, Karnataka \\ 
\texttt{amshar@microsoft.com}}

\newtheorem{theorem}{Theorem}[section]
\newtheorem{corollary}{Corollary}[theorem]
\newtheorem{lemma}[theorem]{Lemma}
\theoremstyle{definition}
\newtheorem{defn}[theorem]{Definition}
\theoremstyle{note}
\newtheorem{note}[theorem]{Note}

\begin{document}
\maketitle

\begin{abstract}
    Estimating average causal effect (ACE) is useful whenever we want to know the effect of an intervention on a given outcome.  In the absence of a randomized experiment, many methods such as stratification and inverse propensity weighting have been proposed to estimate ACE. However, it is hard to know which method is optimal for a given dataset or which hyperparameters to use for a chosen method. 
 To this end, we provide a framework to characterize the loss of a causal inference method against the true ACE,  by framing causal inference as a representation learning problem. 
  We show that many popular methods, including back-door methods can be considered as weighting or representation learning  algorithms, and provide general error bounds for their causal estimates. In addition, we consider the case when unobserved variables can confound the causal estimate and extend proposed bounds using  principles of robust statistics, considering confounding as contamination under the Huber contamination model.  These bounds are also estimable; as an example, we provide empirical bounds for the Inverse Propensity Weighting (IPW) estimator and show how the bounds can be used to optimize the threshold of clipping extreme propensity scores. Our work provides a new way to reason about competing estimators,  
  and opens up the potential of deriving new methods by minimizing the proposed error bounds.    
\end{abstract} 

\section{Introduction}
\label{introduction}
Consider the canonical causal inference problem where the goal is to find the effect of a treatment $T$ on some outcome $Y$, as shown in the structural causal model in  Figure~\ref{fig:source}. This is relevant for estimating the effect of any fixed intervention, such as setting a system parameter, a medical intervention~\cite{hernanbook} or a policy in social science settings~\cite{morgan2015counterfactuals}.  Here $W$ and $U$ are observed and unobserved common causes respectively, which affect the observed conditional distribution $\Pr(Y|T)$. To estimate the causal effect of $T$, methods typically condition on the observed common causes $W$ using the ``back-door'' formula \cite{pearl2009book}, including methods such as stratification \cite{lunceford2004stratification}, matching \cite{rubin1996matching}, and inverse weighting \cite{rosenbaum1983central}. All of these methods work under the ``ignorability" or the ``selection on observables" assumption, where $U$ is assumed to have no effect once we condition on $W$ (i.e. $\Pr(Y|T,W) = \Pr(Y|T,W,U)$). 
In practice, however, ignorability is seldom satisfied and its violation can lead to significant errors, even changing the direction of the effect estimate. Because $U$ is unobserved, current methods provide no bounds on the error in a causal effect estimate when the assumption is violated. This makes it hard to compare methods for a given dataset, or to assess sensitivity of an estimate to unobserved confounding, except by simplistic simulations of the effect of $U$ \cite{rosenbaum2002observational}. 

In this paper, we provide a general framework for estimating error for causal inference methods, both in the presence and absence of $U$. Our insight is that the causal inference problem can be framed as a domain adaptation~\cite{mansour2009domain} problem, where the target distribution is generated from a (hypothetical) randomized experiment on $T$, as shown in Figure~\ref{fig:target1}.  
Under this target distribution $P$,  the observed effect $P(Y|T)$ is the same as the causal effect, $P(Y|do(T))$ since $T$ is no longer affected by $W$ or $U$~\cite{pearl2009book}. The goal of causal inference  then is to use data from a source distribution $Q$ and estimate a function that approximates $P(Y|T)$. Alternatively, one can consider this as a task of learning an intermediate distribution R (or a \emph{representation}), such that $R(Y|T)$ will be as close as possible to $P(Y|T)$.  In this paper, using the lens of domain adaptation~\cite{mansour2009domain},   we provide bounds on the error of such estimators for the average causal effect (ACE), based on distance (bias) of the intermediate distribution $R$ from $P$ and variance in estimating it. In particular, we show that many causal inference methods such as stratification and inverse propensity weighting (IPW) can be considered as learning an intermediate representation. 

When $U$ is ignorable, we provide bounds that separate out the effects of bias and variance in choosing $R$ and derive a procedure to estimate them from data. Empirical simulations show the value of the proposed error bound in evaluating different intermediate representations, and correspondingly, causal inference algorithms. For instance, our bound can be used to select the optimal threshold for clipping extreme probabilities---a common technique in weighting algorithms such as IPW---in order to minimize error.   When $U$ is not ignorable, we utilize theory from robust estimators to characterize $U$'s effect on $Y$. The intuition is that confounding effect of $U$ on $Y$ can be considerd as contamination (noise) added to true function between $T$ and $Y$. In addition, we assume that this noise affects only a fraction of input data rows. Such an assumption is plausible whenever effect of $U$ is specific to certain units, for example, unobserved genes may only affect some people's health outcome and be ignorable for other people.  We use the Huber-contamination model~\cite{huber1992robust} to model this noise, provide a robust estimator for the causal effect~\cite{lai2016agnostic}, and bound its error under the assumption that $U$ only affects a fraction of all outcomes $Y$. When such an assumption is not plausible, the bounds still allow us to study the sensitivity of the error as the amount of contamination (confounding) by $U$ is changed. Overall, our error bounds on causal estimators provides a principled way to compare different estimators and conduct sensitivity analysis of causal estimates with minimal parametric assumptions. 

\begin{figure}[tb]
\centerline
{
    \subfloat[Source distribution Q]
{
    \includegraphics[width=0.25\linewidth]{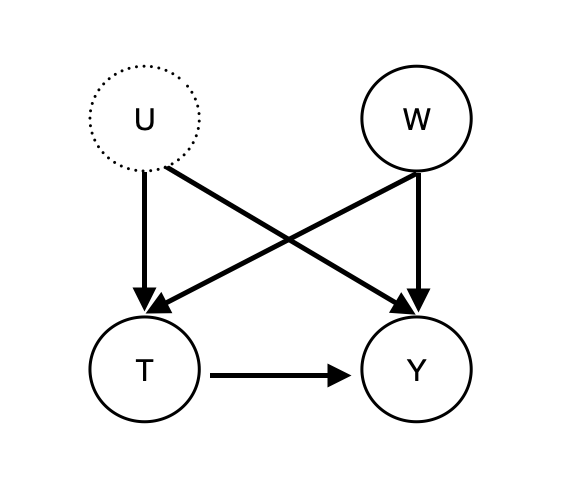}
    \label{fig:source}
}
   \subfloat[Target Distribution P]
{
    \includegraphics[width=0.25\linewidth]{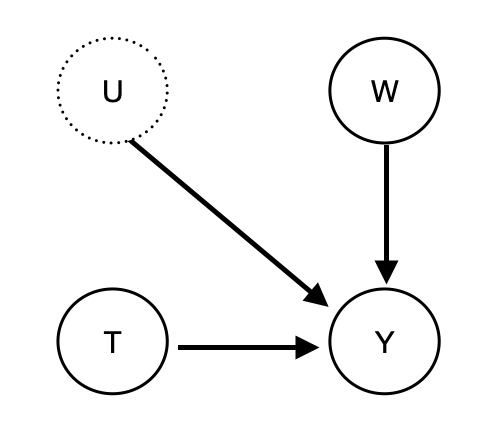}
    \label{fig:target1}
}
    \subfloat[Source Distribution]
{
    \includegraphics[width=0.25\linewidth]{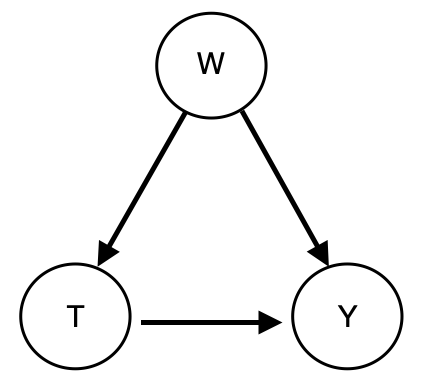}
    \label{fig:source-simple}
}
\subfloat[Target Distribution]
{
    \includegraphics[width=0.25\linewidth]{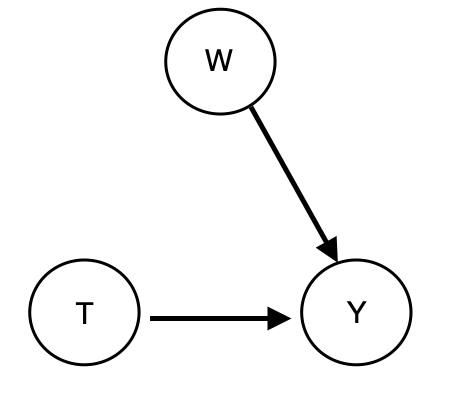}
    \label{fig:target-simple}
}
}
\caption{Causal graphical models denoting source and target distributions in the presence and absence of unobserved confounders $U$.}
\end{figure}

\input{background}

\input{rep-learning}

\input{bounds-no-u}

\input{bounds-robust}

\input{evaluation}
\input{related}

\section{Conclusion}
We have provided general error bounds for any causal estimator that can be written as a weighted representation learner. The error naturally decomposes into  the sampling error in estimating $R$ and measure of distance between the weighted distribution and target distribution $P$. The error terms also yield important insights for developing new methods by minimizing the error bounds.

\medskip

\small
\bibliographystyle{acm}
\bibliography{ref.bib}

\section{Supplementary Materials}
\appendix
\input{appendix}
\end{document}

%% file: background.tex
\section{Background \& Contributions}

\subsection{Defining predictive and causal effect}
\label{pred-causal-effect}
We first define the average causal effect (ACE) and show its connection to the average predictive effect. Let $V=\{W,T,Y\}$ be the set of observed variables. $T$ represents the treatment variable and  $Y$ the outcome variable. $W$ represents the set of all observed common causes of $T$ and $Y$, and $U$ denotes the set of all \emph{unobserved} common causes of $T$ and $Y$. 
Throughout, we assume that the treatment is binary, $T\in \{0,1\}$, where $T=1$ denotes that a treatment was assigned and otherwise for $T=0$. $Y$ and $W$ can be discrete or continuous. $U$ are unobserved common causes and we make no assumptions about them. Figure~\ref{fig:source} shows this  observed data distribution as the \emph{source distribution} $Q$, using the structural causal graph~\cite{pearl2009book} notation. Vertices represent variables and edges represent the potential causal link between these variables.

Under the source distribution $Q$, we define the average predictive effect (APE) of $T$ on $Y$ as:
\begin{equation}\label{eq:APE}
    APE_Q = E_Q[Y|T=1] - E_Q[Y|T=0]
\end{equation}
Intuitively, APE captures the correlation between $T$ and $Y$. In general, correlation is not a sufficient condition to imply that the treatment had actually caused the observed outcome. This is because Reichenbach's common cause principle states that if two random variables $T$ and $Y$ are statistically dependent $(T\not\!\perp\!\!\!\perp Y)$, then there exists a third variable, say $W$ that can causally influence both. 
Thus, using the \emph{do-operator}~\cite{pearl2009book}, we can write the average causal effect of T on Y as, 
\begin{align}\label{eq:ACE}
    ACE &\triangleq E[Y|do(T=1)] - E[Y|do(T=0)]
\end{align}
where $do(T=1)$ operator denotes setting the value of $T=1$ independent of all ancestors of $T$ in the causal graph.  A randomized experiment where one randomizes $T$ and then observes effect on $Y$ is  one way for estimating the ACE. Due to randomization, any effect of W or U on T is wiped out and thus the method is considered as a ``gold standard'' for ACE. Effectively, randomization constructs a new distribution $P$ where there are no back-door paths that confound effect of T on Y and thus average predictive effect equals ACE (formally, due to Rule-2 of do-calculus~\cite{pearl2009book}).  We call this the \textit{target distribution} $P$ and write:
\begin{equation}\label{eq:APE}
    APE_P = E_P[Y|T=1] - E_P[Y|T=0] = E[Y|do(T=1)] - E[Y|do(T=0)]= ACE
\end{equation}

\subsection{Causal inference methods}
Without a randomized experiment, however, the ACE cannot be identified from observational data from $Q$. Methods for causal inference typically make the ignorability assumption, implying that U does not any additional effect after conditioning on the effect of W. That is, $\Pr(Y|T,W)=\Pr(Y|T,W,U)$. In graphical language, conditioning on W \emph{``d-separates''} $T$ and $Y$ in a modified graph that has no outgoing edges from $T$. Under this assumption, various methods have been proposed using the ideas of \emph{conditioning} or \emph{weighting}; for a review see \cite{hernanbook,rosenbaum2002observational,rubin1996matching}.

In conditioning-based methods, we separate data into strata based on $W$, estimate the predictive effect in each stratum which is equal to the causal effect,  and then use the back-door formula~\cite{pearl2009book} to aggregate the estimate. This method is called stratification (\emph{matching} when each stratum is of size 1).
\begin{align} 
    Q(Y\vert W, T=1) = P(Y\vert W, T=1) \Rightarrow E(Y\vert do(T=1))=\sum_W E_Q(Y\vert W, T=1)Q(W) 
    \label{eqn:strat}
\end{align}
Alternatively, in weighting-based methods, one can weight samples from the source data Q to resemble a sample from P. In other words, we ensure that the treatment assignment probability $Q(T=1|W)$ matches the target distribution $P$ as far as possible. This is achieved using importance sampling, or a common variant called inverse propensity weighting where each sample point's weight is inversely proportional to its probability of occurence in the data. This weighting gives more weight to samples that do not occur frequently due to effect from $W$, 
thus compensating for selection bias in $Q(T=1)$. Assuming $n$ is the number of samples from Q, we write:
\begin{equation}
\label{eq:IPW}
    \hat{IPW}=\dfrac{1}{n}\bigg(\sum\limits_{i=1}^n\dfrac{T_i*Y_i}{Q(T_i=1\vert W_i)}-\sum\limits_{i=1}^n\dfrac{(1-T_i)*Y_i}{1-Q(T_i=1\vert W_i)}\bigg)
\end{equation}

\subsection{Our contributions}
We make the following contributions:
\begin{itemize} [itemsep=0pt,leftmargin=*,topsep=2pt]
    \item Using the relationship between APE and ACE, we formulate causal inference as the problem of learning a representation R such that  $APE_R$ approximates $APE_P$. Specifically, we use a probability weighting method to construct a representation R, and show that popular methods such as stratification and IPW are special cases of the weighting method. (Section~\ref{sec:causal-rep})
    \item We provide bounds for the loss in estimating ACE as $APE_R$ and separate out the loss incurred due to bias and variance in selecting $R$. We apply these bounds to develop a data-driven method for selecting the clipping threshold of an IPW estimator. (Section~\ref{sec:bounds}) 
    \item When unobserved confounders $U$ may be present, we extend these bounds using recent work in robust estimation and provide the first results that can characterize error in the presence of unobserved confounding. (Section~\ref{sec:robust}) 
\end{itemize}

%% file: rep-learning.tex
\section{Causal Inference as Representation Learning}
\label{sec:causal-rep}
As discussed above, the problem of estimating ACE can be considered as learning the target distribution $P$ given data from $Q$ and then estimating the observed conditional expectation $E_P[Y|T]$. $Q$ can be considered as the factual distribution, and $P$ the counterfactual distribution corresponding to the question---\textit{what would have happened if we intervened on $T$ without changing anything else?} Our goal is to learn an intermediate distribution $R$ that approximates $P$. This setup is similar to domain adaptation, except that  instead of learning a function $f$ as in Mansour et al.~\cite{mansour2009domain},  we learn a new representation of the data and estimate the same APE function.

\subsection{Defining the weighting method}\label{problemsetup}
Given this formulation, a key question is how to generate a representation such that its APE will be close to ACE. We first define  a consistent estimator for $APE$ under any distribution $R$, $h(x_R)$. Then, the estimator ($h$), and the  APE under infinite samples ($h^\infty$) can be written as:
\begin{align}
  h(x_R)=\hat{E}_R[Y\vert T=1]-\hat{E}_R[Y\vert T=0];  &&   h^\infty(x_R)=E_R[Y\vert T=1]-E_R[Y\vert T=0]  
\end{align}\label{eq:hdef}
By the above definition of $h$, $ACE=h^\infty(x_P)$. Next, we define a class of distributions given by weighting of $Q$. Following Johansson et al.~\citep{johansson2018learning}, we generate a weighted representation $R$ from our source distribution $Q$ such that $h(x_R)$ is an estimator for ACE. 
\begin{defn}\label{def:beta}
    Let $Q(W,T,Y)$ be the source distribution. We define a weighting function $\beta(W,T)$ to generate a representation $R$ such that,
    $$\beta(W,T) = \frac{R(W\vert T)}{Q(W\vert T)}= \frac{R(T\vert W)}{Q(T\vert W)}$$ and that $R$ is a valid probability distribution, $\forall W,T\hspace{1ex}R(W,T)\geq 0$, $\sum_W \sum _T R(W,T)=1$. 
\end{defn}

\subsection{IPW and stratification as weighting methods}
We now show that the IPW estimator and back-door methods such as stratification can be considered as a weighted $\beta$ estimator.
\begin{theorem}
\label{thm:IPW}
Consider the causal graphical model in Figure~\ref{fig:source-simple} where the observed common causes $W$ are the only confounders. The IPW estimator can be written as a representation $R$ where $\beta(W,T=t)=\dfrac{R(W|T=t)}{Q(W\vert T=t)}=\dfrac{Q(W)}{Q(W\vert T=t)}$.
\end{theorem}
\begin{proof}
    Here we consider only the $T=1$ part of IPW estimator from the RHS of Equation~\ref{eq:IPW}. The proof is symmetric for $T=0$.
\begin{align}
    IPW_{T=1}&=\dfrac{1}{n}\bigg(\sum\limits_{i=1}^n\dfrac{T_i*Y_i}{Q(T_i\vert W_i)}\bigg)
    =E_Q\bigg[\dfrac{1_{T=1}Y}{Q(T=1\vert W)}\bigg]
    =\sum_W \dfrac{E_Q[Y \vert T=1,W]Q(T=1\vert W)Q(W)}{Q(T=1\vert W)}\\
    &=\sum_W\sum_Y\dfrac{YQ(Y\vert T=1,W)Q(T=1\vert W)Q(W)}{Q(T=1\vert W)}
    =\sum_W\sum_Y YQ(Y\vert T=1,W)Q(W) \label{eqn:backdoor-ipw}
\end{align}
where $1_{T=1}$ is an indicator function that is 1 whenever T is 1 and 0 otherwise. The second equality above utilized that $T$ is binary.

Similarly, we can write the the first part $(T=1)$  of the $APE$ under $R$ as:
\begin{align*}
    APE_{T=1}^R=\sum_Y YR(Y\vert T=1)
    =\sum_W\sum_Y YR(Y\vert T=1,W)R(W\vert T=1)
    =\sum_W\sum_Y YQ(Y\vert T=1,W)R(W\vert T=1)
\end{align*}

where the last equality is since $Q(Y\vert T=1,W)=R(Y \vert T=1,W)$ (\emph{ignorability} assumption from Equation~\ref{eqn:strat}). Further, using $\beta(W,T) Q(W\vert T=1)=R(W\vert T=1)$ (by definition),
\begin{align*}
    APE_{T=1}^R=
    =\sum_W\sum_Y YQ(Y\vert T=1,W)\beta(W,T=1)Q(W\vert T=1)
\end{align*}
Comparing the two terms for $IPW_{T=1}$ and $APE_{T=1}^R$, if     $\beta(W,T=1)=\dfrac{Q(W)}{Q(W\vert T=1)}$, then $IPW_{T=1}=APE_{T=1}^R$.
\end{proof}

The  above proof also shows the equivalence of IPW and backdoor-based stratification~\cite{hernanbook}. Under the conditions of Theorem~\ref{thm:IPW}, and using $\sum_W E_Q(Y\vert W, T=1)Q(W)=\sum_W\sum_Y YQ(Y\vert T=1,W)Q(W)$, we have: 
\begin{corollary}
    The stratification estimator from Equation~\ref{eqn:strat}, $\sum_W E_Q(Y\vert W, T=1)Q(W)$ 
    is equivalent to Equation~\ref{eqn:backdoor-ipw} and thus also a weighting method with $\beta(W,T=1)=\dfrac{Q(W)}{Q(W\vert T=1)}$.
\end{corollary}

%% file: bounds-no-u.tex
\section{Bounds for ACE without unobserved confounders}\label{sec:bounds}
\label{sec:general-bounds}
Let us first consider a setting where the latent confounder $U$ has no effect on $T$ or $Y$. That is,  the treatment $T$ is assigned to a unit according to only observed covariates $W$ (shown in Figure~\ref{fig:source-simple}). 

Based on this assumption, we showed that a causal inference method can be characterized by a weighted distribution $R$ that it outputs. We now provide error bounds based on a given distribution $R$. 
We use a setup similar to that of Mansour et al.~\cite{mansour2009domain}, where the loss function $L$ is assumed to be  symmetric and that it follows the triangle inequality. Common loss functions such as the $L1$ and $L2$ loss satisfy these properties. We are interested in the loss between an estimated effect $h(x_R)$ and the ACE, $h^\infty(x_P)$.   
If the loss function is assumed to be $L1$, the loss can be defined as: 
$L({h}(x_R),h^\infty(x_P))\triangleq \vert {h}(x_R)-h^\infty(x_P) \vert$

\subsection{Loss Bound: A tradeoff between bias and variance}
\label{bounds}
Before we state the loss bounds, we define two terms that characterize the loss. Intuitively, if $R$ is chosen to be similar to Q ($\beta \approx 1$), then $h(x_R)$ will have low sample variance as the weights will be bounded, but high bias since $h^\infty(x_R)$ may be very different from the ACE, $h^\infty(x_P)$.  Conversely, if we choose $R$ to be close to $P$, then $h(x_R)$ will have low bias error, but possibly high variance as the $\beta$ weights can be high. Thus, for any $R$, the error is a combination of these factors: bias in choosing $R$, and the variance in estimating $h(x_R)$.     

To capture the error due to bias, we define a weighted L1 distance between $R$ and $P$.

\begin{defn}(Weighted L1 Distance)
    Assume $R$ and $P$ are distributions over $W,T,Y$. We define the weighted L1 distance(WLD), between $R,P$ as follows:
    \begin{equation}\label{eq:WLD}
        WLD_{T=t}(R,P)=\sum_W (R(W\vert T=t)-P(W\vert T=t)) E_{Q}[Y\vert T=t,W]
    \end{equation}
\end{defn}
We also define a $\textit{VR}$ term due to variance  in estimation.$\textit{VR}_{T=t}= \alpha_{T=t}(\hat{Q}, \hat{\beta}) - \alpha_{T=t}(Q,\beta)$.
\begin{defn}\label{def:alpha}(Sample Error Terms)
Define 
\begin{equation}
    \alpha_{T=t}(\hat{Q},\hat{\beta}) \triangleq \sum_W \hat{\beta}\hat{Q}(W\vert T=t)\sum_Y Y\hat{R}(Y\vert T=t,W)
\end{equation}
    
Using the same notation, population $\alpha$ is defined as
\begin{equation}
    \alpha_{T=t}(Q,\beta) \triangleq \sum_W \beta Q(W\vert T=t)\sum_Y YR(Y\vert T=t,W)    
\end{equation}

\end{defn}
\begin{note}\label{Causal Invariance}
The causal mechanism does not change across the distributions $P,Q,R$, which means, $P(Y\vert T,W)=Q(Y\vert T,W)=R(Y \vert T,W)$
\end{note}
For ease of exposition, we'll assume the loss function is L1. We have the following result. 
\begin{theorem}
Assume that the loss function $L$ is symmetric and obeys the triangle inequality. $h$ is a function on a representation $R$ such that $h(x_R)=E_R[Y\vert T=1]-E_R[Y\vert T=0]$. Then, for any valid weighted representation $R$, if there are no unobserved confounders and and $L=L1$, then:
\begin{align*}
    L(h(x_R),h^\infty(x_P))\leq & \vert\alpha_{T=1}(\hat{Q},\hat{\beta})- \alpha_{T=1}(Q,\beta)\vert+\vert\alpha_{T=0}(\hat{Q},\hat{\beta})-\alpha_{T=0}(Q,\beta)\vert \\
     & +\vert WLD_{T=1}(\beta Q,P) \vert + \vert WLD_{T=0}(\beta Q,P)\vert
\end{align*}

\end{theorem}
The proof is in Supplementary Materials. 

\subsection{Estimating the loss bound from observed data}
Given a causal inference algorithm (as defined by its weights $\beta$), we now describe how to estimate these bounds from data. 

\paragraph{Estimating VR term} For $\textit{VR}$ term, we use McDiarmid's inequality~\cite{raginsky2013concentration}. We can rewrite $\alpha_{T=t}$ as:
\begin{align}
 \sum_W \beta{Q}(W\vert T=t)\sum_Y Y{R}(Y\vert T=t,W)
= \sum_W \sum_Y \beta Y Q(Y,W|T=t) = \mathbb{E}_{Q(Y,W|T=t)}\beta Y
\end{align}
where we used that $R(Y\vert T=t, W)=Q(Y\vert T=t,W)$. Thus, $\textit{VR}_{T=t}$ can be written as an expected value. Then estimated $\hat{\alpha}_{T=t}$ can be written as $\frac{1}{N_{T=t}}\sum_{i=0}^{N_{T=t}}\hat{\beta_i} Y_i$. Since $g(X) = Y\beta$ is a function of i.i.d samples $X=(W,T, Y)$, we can apply the McDiarmid inequality,
$$\Pr[g(X_n) - \mathbb{E}(g(X_n)) \leq t] \geq 1 - exp(-\frac{2t^2}{\sum_i^n c_i^2})$$
where $c_i$ is the maximum change in $g(X_n)$ after replacing $X_i$ with another value $X_i'$. We compute a data-dependent bound for each $c_i$ by considering all possible discrete values for $X'_i$ and computing the resultant difference in $g$. We provide the code to estimate $c_i$ in \url{github/anonymizedcode}. 

Fixing the RHS as $p$, we obtain $t = \sqrt{\frac{\sum_i^n c_i^2 \log\frac{1}{1-p}}{2}}$. Thus, we can estimate the difference $\textit{VR}_{T=1}$ as 
\begin{align}
    \texttt{With} \Pr=p & &    | \alpha_{T=1} - \hat{\alpha}_{T=1} | \leq \sqrt{\frac{\sum_i^n c_i^2 \log\frac{1}{1-p}}{2}}
\end{align}

\paragraph{Estimating WLD}
For some estimators like IPW, we can prove that they are unbiased and hence $WLD_{T=t}=0$. 
\begin{lemma}
    For the IPW estimator, if $P(W)=Q(W)=R(W), L(h^\infty(x_R),h^\infty(x_P))=0$.
\end{lemma}
Proof is in Supplementary Materials. For others, our estimation depends on assuming that $Q(T=1|W)$ is bounded between $[\rho, 1-\rho]$ for some sufficiently small $\rho$. The intuition is that assignment of $T$ depends on $W$, but for every $W=w$ there is a minimum probability that $T=1$ or $T=0$. This assumption can be stated as ``no extreme selection based on W'' and is a generalization of the \emph{overlap} assumption~\cite{shalit2017estimating}, a requirement for IPW and other causal inference methods. Under this assumption, WLD can be written as:
\begin{align}
    WLD_{T=t}(R,P) = \sum_W ( R(W\vert T=t)-P(W\vert T=t) ) E_{Q}[Y\vert T=t,W] \\
    =\sum_W \sum_Y (R(W\vert T=t)-P(W)) Y Q(Y\vert T=t,W) \\
=\sum_W \sum_Y Y R(W\vert T=t) Q(Y\vert T=t,W) -\sum_W \sum_Y YQ(W)  Q(Y\vert T=t,W) \\
=\sum_W \sum_Y Y\beta Q(Y, W\vert T=t) -\sum_W \sum_Y Y\beta^*Q(W\vert T=t)  Q(Y\vert T=t,W)  \label{eqn:wld-est-rhs}
\end{align}
where the third equality is due to $P(W|T=t)=P(W)=Q(W)$ and the fourth due to the definition of $\beta$ from above. Here $\beta$ corresponds to a causal inference method given by the representation R and $\beta^*$ corresponds to unbiased IPW weights, estimated by using IPW and then   clipping propensity scores as $\min(\rho, \hat{Q}(T=1|W))$ (assuming bounded $Q(T=1|W)$).    The first term of Equation~\ref{eqn:wld-est-rhs} can be estimated as as $\mathbb{E}_Q (Y\beta|T=1)$ and the second term as $\mathbb{E}_Q(Y\beta^*|T=1)$. 
We show applications of estimating these bounds in Section~\ref{sec:eval}.

%% file: bounds-robust.tex
\section{Bounds for ACE with unobserved confounders}\label{sec:robust}
We now provide bounds for the general case of causal inference in the presence of unobserved confounders. 
Let $V=\{W,T,Y,U\}$,  where $W,T,Y$ are the same as before, but $U$ is introduced. 

Our insight is that principles of robust statistics can be used to bound the loss due confounding by $U$. 
Let us consider the example from Section \ref{introduction} where $U$ are unobserved genes that affect the outcome $Y$ as well as the choice of treatment $T$.
In many cases, it can be reasonable to assume that $U$ will affect the outcome $Y$ for only a subset of the population (especially so when the outcome has discrete levels). 
Specifically, we make an assumption that $U$ does not change the outcomes for all the units, instead only for a fraction of units $\eta$.  This assumption can be written in terms of the Huber contamination model~\cite{huber1992robust}, where $U$'s effect is the \emph{contamination} in observed $Y$.  
Formally, we can write,
$$Y \sim (1-\eta)Q(Y\vert T,W)+\eta Q(Y\vert T,W,U)$$
where $\eta$ is the contaminated fraction of samples.
Further, we assume $U$ to be adversarial in nature as described in \cite{lai2016agnostic}, i.e. $U$ is allowed to observe values of $W,T$ and change the value of $Y$ accordingly. 

Under these settings, we show that it is possible to bound $L(h(x_R), h^\infty(x_P))$ by estimating $E_W E_Y[Y\vert T,W]$ robustly and plugging in the additional error due to contamination. 
In effect, this amounts to a two-step procedure: learn a new representation $Q_B$ robustly from distribution $Q$ and then learn $\beta(W,T)$ on this representation $Q_B$ (i.e., weight $Q_B$ to get $R$). In practice, since the bounds from Section~\ref{sec:bounds} only depend on $E[Y|T,W]$, we do not need to estimate $Q_B$ but rather just a robust estimate of the conditional means for $Y|T,W$.
Estimating $E_{Q_B}[Y|T,W]$ with a robust estimator implies removing the backdoor path as in Figure~\ref{fig:target1} and thus, the error of the estimate can be bounded given a contamination fraction $\eta$. The proof proceeds in an analogous way to the previous Section; we next show an application of the bound by estimating error for the IPW estimator. 

\subsection{Bounds for IPW under unobserved confounding}
\label{sec:ipw-bounds}
Recall from Theorem \ref{thm:IPW}, we have $\beta(W,T)=\dfrac{Q(W)}{Q(W\vert T=1)}$ for IPW. To provide a concrete bound, we use the robust mean estimator from Lai et al.~\cite{lai2016agnostic} for $Y$ 
and assume that the fourth moment of $Y$ is bounded, $E((Y-\mu)^4|T,W)\leq C\sigma^4$ where  $\sigma$ is the standard deviation and $C$ is some constant.
We assume $\eta$ fraction contamination (confounding due to $U$) and $\epsilon$ is a parameter for the running time of the robust mean algorithm.

\begin{note}
Define $\gamma \triangleq \sum_W Q(W_i)O(C^1/4(\eta+\epsilon)^3/4\sigma)$
\end{note}

\begin{theorem}
Assume that the loss function $L$ is symmetric and obeys the triangle inequality. $h$ is a function on a representation $R$ such that $h(x_R)=E[Y\vert T=1]-E[Y\vert T=0]$. Then, for any valid weighted representation $R$,if $U\neq \phi$, the following holds with probability $(1-1/poly(n))^{2\vert W\vert}$.

$L(h(x_R),h^\infty(x_P))\leq WLD_{T=1}^W(Q_B,P) + WLD_{T=0}^W(Q_B,P) + \vert\alpha_{T=1}(\hat{Q_B},\hat{\beta_B})- \alpha_{T=1}(Q,\beta_B)-\gamma\vert+\vert\alpha_{T=0}(\hat{Q_B},\hat{\beta_B})-\alpha_{T=0}(Q_B,\beta_B)-\gamma\vert$
\end{theorem}
where $Q_B$ is the ``robust'' version of the distribution $Q$. 
The proof is in Supplementary Materials.
\begin{corollary}
For IPW estimator, if $P(W)=Q(W)=R(W), L(h^\infty(x_R),h^\infty(x_P))=2\gamma$
\end{corollary}
The proof is in Supplementary Materials. Note that depending on the nature of corruption and the adversary model, different robust estimation methods can be used which may provide tighter bounds.

%% file: evaluation.tex
\section{Evaluation: Applying the loss bounds}\label{sec:eval}
We now evaluate our  bounds on simulated data and describe their utility for choosing hyperparameters for causal inference. When there are unobserved confounders, we also propose a new method, \emph{robust IPW} that relies on a robust estimator.  
\paragraph{When $U$ is ignorable ($\tau=\kappa=0$).}
We generate data using the following structural equations:
\begin{align*}
    w_j \sim Binomial(p=0.5) \forall j \in [1,|W|]; \ \ \ \ \   u \sim Normal(\mu,\sigma)       \\
    t = Bernoulli(p=sigmoid(\psi \cdot w +\kappa u));     y = Bernoulli(p=sigmoid(\nu \cdot w + \tau u + \lambda t)) 
\end{align*}
where $\psi, \nu \in \mathbb{R}^M$ and $\lambda$, $\tau$, $\kappa$ are scalar. $T$ is always binary. We chose this formulation since generation of $T$ maps directly to logistic regression, which makes it easy to estimate propensity scores when computing the causal estimate. The true ATE can be obtained by simulating $y_{counterfactual}$ by setting $t=1-t$ in the equation for $y$ above and computing the average difference. We present results for $|W|=5$.

In Figure~\ref{fig:l1bounds}, we  show that the bounds correctly follow the IPW estimate over different levels of confounding by W (values of $\psi$), and different sample sizes. Since IPW is unbiased, the bounds effectively estimate the variance of the estimator: as $\psi$ increases, the error bound is expected to increase. The empirical error is the L1 distance between the actual IPW estimate and the true ATE. Across sample sizes and different values of $\psi$, we find that the proposed bound tracks the empirical error in the IPW estimate (Figure~\ref{fig:l1bounds}).
\begin{figure}
    \includegraphics[scale=0.45]{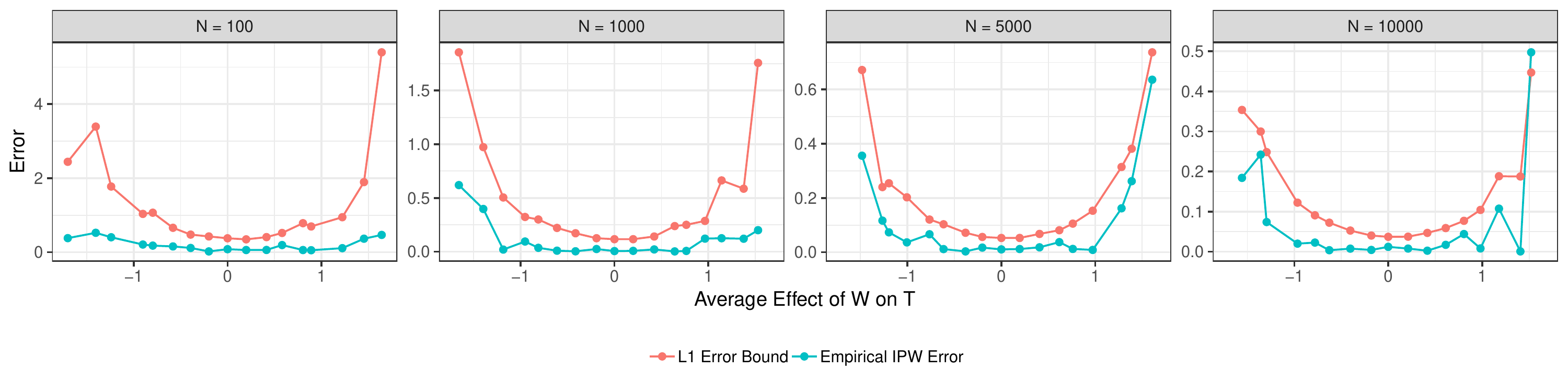}
    \caption{L1-error bound and IPW estimate for different levels of confounding by W. }
    \label{fig:l1bounds}
\end{figure}
\begin{figure}
    \includegraphics[scale=0.45]{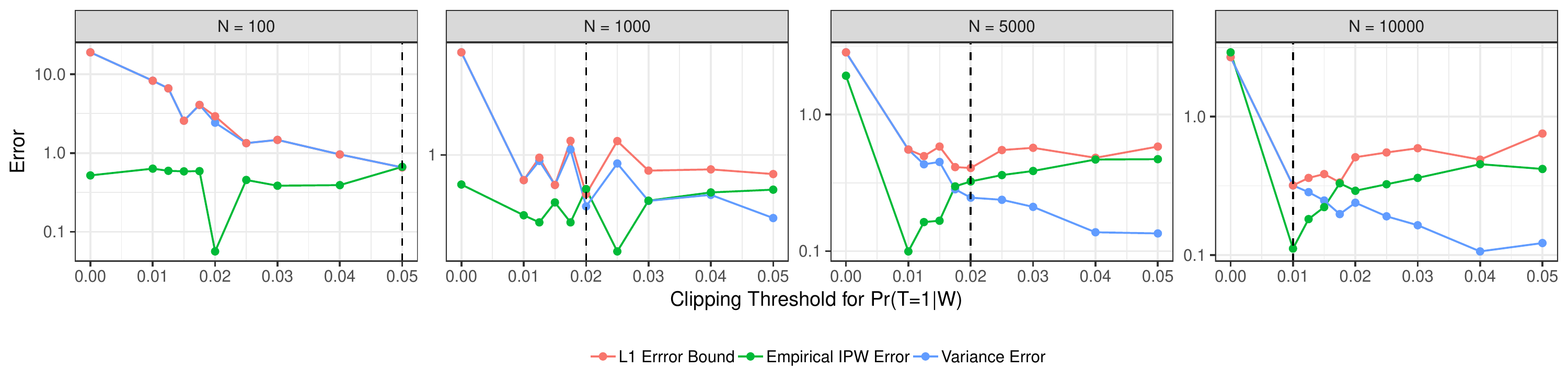}
    \caption{Choosing the clipping threshold for IPW propensity that minimizes L1-error bound.}
    \label{fig:clip-ipw}
\end{figure}

These bounds can have practical significance in choosing hyperparameters in causal inference methods. For instance, consider the popular technique of clipping extremely high propensity scores~\cite{lee2011weight} to reduce IPW variance. This introduces bias in the estimator, and an important question is how to select the clipping threshold. We estimate the loss bound for IPW under different values of the threshold (Figure~\ref{fig:clip-ipw}). To estimate the WLD term, we generate treatment $T$ such that the true probability is bounded between $[\rho,1-\rho]$ where $\rho=0.01$ as discussed in Section~\ref{sec:bounds}. We find that the optimal (one that minimizes the loss bound) clipping threshold varies with sample size, marked by the dotted vertical line. Optimal threshold decreases with sample size: a higher threshold reduces the variance in smaller sample sizes, which is less required as sample size increases.

\paragraph{When U is not ignorable.}
Finally we consider the setting when there are unobserved confounds. Based on the bounds, we propose a robust version of IPW using the estimator from \cite{lai2016agnostic} and evaluate for a continuous $Y$. For a fixed contamination ($\eta$), our proposed robust IPW recovers the true estimate up to an error, and its error increases as $\eta$ is increased. Critically, the variance in the estimator is substantially lower than the standard IPW estimator, but it is biased.  At $\eta=0.05$, for a true causal effect of 1, we obtain an error of $\approx 0.2$. Details are in the Supplementary Materials.

%% file: related.tex
\section{Related Work}
\label{relatedwork}
Our work is related to domain adaptation and representation learning for causal inference. In the domain adaptation problem, the goal is to learn a function that generalizes from a source distribution $Q$ to a target distribution $P$. Mansour et al.~\citep{mansour2009domain} provided bounds for generalization of a function between distributions and proposed weighting as a technique to minimize distance between source and target distributions. Gretton et al. \cite{gretton2009covariate} and Kallus et al.~\cite{kallus2016generalized} have also proposed methods to learn weights from data samples so that the distance between the weighted source and the target is reduced. 
Weighting of the source distribution can be considered as learning a representation. Based on this idea, Johansson et al. \citep{johansson2016learning} proposed a domain adaptation framework to learn a counterfactual distribution from the factual distribution. The estimated counterfactual distribution is then used to evaluate causal effect conditional on specific covariates, also known as the conditional average treatment effect (CATE). 

For estimating ACE, there is a rich literature in statistics that proposes estimators based on the backdoor formula, including stratification, matching and propensity score-based methods like IPW \citep{shah2005propensity}. In the absence of latent confounders, error in estimating causal effect has been well studied for estimators like IPW \cite{rosenbaum1983central} . For instance, estimators like Horvitz Thompson and Hajeck estimators \cite{henderson2013estimating} provide us with a unbiased variance estimate for IPW. 
However, all of the above methods for CATE and ACE do not focus on producing general error bounds and  assume that $U$ is ignorable.

%% file: appendix.tex
\section{Proof of Theorem 4.4}
\begin{defn}(Weighted L1 Distance)
    Assume $R$ and $P$ are distributions over $W,T,Y$. We define the weighted L1 distance (WLD), between $R,P$ as follows:
    \begin{equation}\label{eq:WLD}
        WLD_{T=t}(R,P)=\sum_W (R(W\vert T=t)-P(W\vert T=t)) E_{Q}[Y\vert T=t,W]
    \end{equation}
\end{defn}
And similarly, we define a VR term due to variance  in estimation.$VR_{T=t}= \alpha_{T=t}(\hat{Q}, \hat{\beta}) - \alpha_{T=t}(Q,\beta)$.
\begin{defn}\label{def:alpha}(Sample Error Terms)
Define 
\begin{equation}
    \alpha_{T=t}(\hat{Q},\hat{\beta}) \triangleq \sum_W \hat{\beta}\hat{Q}(W\vert T=t)\sum_Y Y\hat{R}(Y\vert T=t,W)
\end{equation}
    
Using the same notation, population $\alpha$ is defined as
\begin{equation}
    \alpha_{T=t}(Q,\beta) \triangleq \sum_W \beta Q(W\vert T=t)\sum_Y YR(Y\vert T=t,W)    
\end{equation}

\end{defn}
\begin{defn}\label{def:beta}
    Let $Q(W,T,Y)$ be the source distribution. We define a weighting function $\beta(W,T)$ to generate a representation $R$ such that,
    $$\beta(W,T) = \frac{R(W\vert T)}{Q(W\vert T)}= \frac{R(T\vert W)}{Q(T\vert W)}$$ and that $R$ is a valid probability distribution, $\forall W,T\hspace{1ex}R(W,T)\geq 0$, $\sum_W \sum _T R(W,T)=1$. 
\end{defn}
\begin{note}\label{Causal Invariance}
The causal mechanism does not change in the distributions $P,Q,R$, which means, $P(Y\vert T,W)=Q(Y\vert T,W)=R(Y \vert T,W)$
\end{note}
For ease of exposition, we'll assume the loss function is L1. We have the following result. 

\begin{theorem}
Assume that the loss function $L$ is symmetric and obeys the triangle inequality. $h$ is a function on a representation $R$ such that $h(x_R)=E_R[Y\vert T=1]-E_R[Y\vert T=0]$. Then, for any valid weighted representation $R$, if $U= \phi$ and $L=L1$, the following holds
\begin{align*}
    L(h(x_R),h^\infty(x_P))\leq & \vert\alpha_{T=1}(\hat{Q},\hat{\beta})- \alpha_{T=1}(Q,\beta)\vert+\vert\alpha_{T=0}(\hat{Q},\hat{\beta})-\alpha_{T=0}(Q,\beta)\vert \\
     & +\vert WLD_{T=1}(\beta Q,P) \vert + \vert WLD_{T=0}(\beta Q,P)\vert
\end{align*}

\end{theorem}
\begin{proof}

By Triangle Inequality,
\begin{equation}
\label{eq:TriangleInequality}
    L(h(x_R),h^\infty(x_P))\leq L(h(x_R),h^\infty(x_R))+L(h^\infty(x_R),h^\infty(x_P))
\end{equation}
\paragraph{PART I.}
Consider the second term in the RHS, $L(h^\infty(x_R),h^\infty(x_P))$
\begin{align*}
    &=L((E_R[Y\vert T=1]-E_R[Y\vert T=0])-(E_P[Y\vert T=1]-E_P[Y\vert T=0]))\\
    &=L((E_R[Y\vert T=1]-E_P[Y\vert T=1])+(E_P[Y\vert T=0]-E_R[Y\vert T=0]))\\
\end{align*}
Expanding on the first term,  $(E_R[Y\vert T=1]-E_P[Y\vert T=1])$
\begin{align*}
 \label{eq:WLD}
    &=\sum_Y Y(R(Y\vert T=1)-P(Y\vert T=1))\\
    &=\sum_Y Y{\sum_W R(Y\vert T=1,W)R(W\vert T=1)-P(Y\vert T=1,W)P(W\vert T=1)}\\
    &=\sum_W\sum_Y YR(Y\vert T=1,W)(R(W\vert T=1))-(P(W\vert T=1))\ldots \text{By}\ref{Causal Invariance}\\
    &=\sum_W(R(W\vert T=1)-P(W\vert T=1))\sum_Y YR(Y\vert T=1,W)\ldots\text{By}\ref{Causal Invariance}\\
    &=\sum_W(R(W\vert T=1)-P(W\vert T=1))\sum_Y YQ(Y\vert T=1,W)\ldots\text{By}\ref{Causal Invariance}
    \end{align*}
    Now using the definition of Expectation and $\beta(W,T)$ (Definition~\ref{def:beta})
\begin{align*}
    &=\sum_W(R(W\vert T=1)-P(W\vert T=1))E_{Q}[Y\vert T=1,W]\\
    &=\sum_W(\beta(W,T=1)Q(W\vert T=1)-P(W\vert T=1))E_{Q}[Y\vert T=1,W]\\
    &=WLD_{T=1}(\beta Q,P)
\end{align*}

Similarly expanding$(E_R[Y\vert T=0]-E_P[Y\vert T=0])$ (by symmetry)
\begin{align*}
&\sum_W(\beta(W,T=0)Q(W\vert T=0)-P(W\vert T=0))E_{Q}^Y[Y\vert T=0,W]\\
&= WLD_{T=0}(\beta Q,P)
\end{align*}

\paragraph{PART II} 
Now let us consider the first part of the RHS of Equation~\ref{eq:TriangleInequality}.

$L(h(x_R),h^\infty(x_R))$
\begin{align*}
    &=L((\hat{E_R}[Y\vert T=1]-\hat{E_R}[Y\vert T=0])-(E_R[Y\vert T=1]-E_R[Y\vert T=0]))\\
    &=L((\hat{E_R}[Y\vert T=1]-E_R[Y\vert T=1])+(E_R[Y\vert T=0]-\hat{E}_R[Y\vert T=0]))\\
\end{align*}
Solving for $(\hat{E_R}[Y\vert T=1]-E_R[Y\vert T=1])$
\begin{align*}
    &=\sum_Y Y(\hat{R}(Y\vert T=1)-R(Y\vert T=1))\\
    &=\sum_Y{\sum_W Y\hat{R}(Y\vert T=1,W)\hat{R}(W\vert T=1)-YR(Y\vert T=1,W)R(W\vert T=1)}\\
    &=\sum_W\sum_Y Y\hat{R}(Y\vert T=1,W)\hat{R}(W\vert T=1)-\sum_W{\sum_Y YR(Y\vert T=1,W)R(W\vert T=1)}\\
    &=\sum_W \hat{R}(W\vert T=1)\sum_Y Y\hat{R}(Y\vert T=1,W)-\sum_W R(W\vert T=1)\sum_Y YR(Y\vert T=1,W)\\
    &=\alpha_{T=1}(\hat{Q},\hat{\beta})-\alpha_{T=1}(Q,\beta)
\end{align*}
where the last equality follows from Definition~\ref{def:alpha}.

Similarly expanding $\hat{E}_R[Y\vert T=0]-E_R[Y\vert T=0]$ (by symmetry)

\hspace{4ex}$=\sum_W \hat{R}(W\vert T=0)\sum_Y Y\hat{R}(Y\vert T=0,W)-\sum_W R(W\vert T=0)\sum_Y YR(Y\vert T=0,W)\\$

\hspace{4ex}$=\alpha_{T=0}(\hat{Q},\hat{\beta})-\alpha_{T=0}(Q,\beta)$

\paragraph{PART III}Finally, we derive the result assuming Loss Function is L1.
\begin{align*}
    L(h(x_R),h^\infty(x_R))&=L((\hat{E_R}[Y\vert T=1]-E_R[Y\vert T=1])+(E_R[Y\vert T=0]-\hat{E}_R[Y\vert T=0]))\\
    &=\vert(\hat{E_R}[Y\vert T=1]-E_R[Y\vert T=1])+(E_R[Y\vert T=0]-\hat{E}_R[Y\vert T=0])\vert \\
    &\leq \vert\hat{E_R}[Y\vert T=1]-E_R[Y\vert T=1]\vert+\vert E_R[Y\vert T=0]-\hat{E}_R[Y\vert T=0]\vert \\
    &\leq \vert\alpha_{T=1}(\hat{Q},\hat{\beta})- \alpha_{T=1}(Q,\beta)\vert+\vert\alpha_{T=0}(\hat{Q},\hat{\beta})-\alpha_{T=0}(Q,\beta)\vert \\
\end{align*}
\begin{align*}
    L(h^\infty(x_R),h^\infty(x_P))&=L((E_R[Y\vert T=1]-E_P[Y\vert T=1])+(E_P[Y\vert T=0]-E_R[Y\vert T=0]))\\
    &=\vert(E_R[Y\vert T=1]-E_P[Y\vert T=1])+(E_P[Y\vert T=0]-E_R[Y\vert T=0])\vert \\
    &\leq\vert E_R[Y\vert T=1]-E_P[Y\vert T=1]\vert+\vert E_P[Y\vert T=0]-E_R[Y\vert T=0]\vert\\
    &\leq \vert WLD_{T=1}(\beta Q,P)\vert +\vert WLD_{T=0}(\beta Q,P)\vert
\end{align*}

Hence, we obtain the result:
\begin{align*}
L(h(x_R),h^\infty(x_P))\leq & \vert\alpha_{T=1}(\hat{Q},\hat{\beta})- \alpha_{T=1}(Q,\beta)\vert+\vert\alpha_{T=0}(\hat{Q},\hat{\beta})-\alpha_{T=0}(Q,\beta)\vert \\
& +\vert WLD_{T=1}(\beta Q,P)\vert + \vert WLD_{T=0}(\beta Q,P)\vert
\end{align*}
\end{proof}

\section{Proof of Lemma 4.5}
\begin{lemma}
    For the IPW estimator, if $P(W)=Q(W)=R(W), L(h^\infty(x_R),h^\infty(x_P))=0$.
\end{lemma}
\begin{proof}
For the sake of proof, we'll assume the loss function is L1. 
Since,
\begin{align*}
    L(h_R^\infty,h_P^\infty\vert T=1)&=L((E_R[Y\vert T=1]-E_P[Y\vert T=1])\\
    &=\vert E_R\left[Y\vert T=1\right] -E_P\left[Y\vert T=0\right]\\
    &\leq \sum_w \vert R(W\vert T=1)E_R\left[Y\vert T=1,W\right]-P(W\vert T=1)E_P
    \left[Y\vert T=1,W\right]\vert\\
    &= \sum_w \vert\beta_{ipw}(Q,W) Q(W\vert T=1)E_{Q}\left[Y\vert T=1,W\right]-P(W\vert T=1)E_P\left[Y\vert T=1,W\right]\vert\\
    &\leq \sum_w \vert Q(W)\left(E_P\left[Y\vert T=1,W\right]+\right)-P(W\vert T=1)E_P\left[Y\vert T=1,W\right]\vert\\
    &= \sum_w \vert \big(Q(W)-P(W)\big)E_P\left[Y\vert T=1,W\right]\vert\\
    &=0
\end{align*}
A similar argument can be made for $T=0$ and hence, $L(h^\infty(x_R),L^\infty(x_P))=0$
\end{proof}
\section{Proof of Theorem 5.2}
\begin{theorem}
Assume that the loss function $L$ is symmetric and obeys the triangle inequality. $h$ is a function on a representation $R$ such that $h(x_R)=E[Y\vert T=1]-E[Y\vert T=0]$. Then, for any valid weighted representation $R$,if $U\neq \phi$, the following holds with probability $(1-1/poly(n))^{2\vert W\vert}$.

$L(h(x_R),h^\infty(x_P))\leq WLD_{T=1}^W(Q_B,P) + WLD_{T=0}^W(Q_B,P) + \vert\alpha_{T=1}(\hat{Q_B},\hat{\beta_B})- \alpha_{T=1}(Q,\beta_B)-\gamma\vert+\vert\alpha_{T=0}(\hat{Q_B},\hat{\beta_B})-\alpha_{T=0}(Q_B,\beta_B)-\gamma\vert$
\end{theorem}
Where, $Q_B$ represents the `robust' version of the distribution `Q'. 
\begin{proof}
For the sake of proof, we'll assume the loss function is a L1. 
For IPW, $\beta_B(Q,W)=\dfrac{Q_B(W)}{Q_B(W\vert T=t)}$ and $R(W\vert T=t)=Q_B(W)$
\begin{align*}
    L(h_R^\infty,h_P^\infty)&\leq WLD_{T=1}^W(Q_B,P) + WLD_{T=0}^W(Q_B,P)\ldots From \ref{eq:WLD}
\end{align*}
From \cite{lai2016agnostic}, we know with probability $1-1/poly(n)$, $\vert\hat{\mu}-\mu\vert\leq O(C_4^1/4(\eta+\epsilon)^3/4\sigma)$.
Consider, $\sum_Y YQ_B(Y\vert T=1,W=w_i)$. This is evaluating conditional mean of $Y$ robustly for a fixed value of $w_i$.

For a fixed value of $W=w_i, T=t$, with probability $1-1/poly(n)$
\begin{equation*}
    \bar{Y}\leq \hat{\bar{Y}}+O(C_4^1/4(\eta+\epsilon)^3/4\sigma)
\end{equation*}
Since, $Q_B(w_i)>0$,
\begin{equation*}
    Q_B(w_i)\bar{Y}\leq Q_B(w_i)\{\hat{\bar{Y}}+O(C_4^1/4(\eta+\epsilon)^3/4\sigma)\}
\end{equation*}
Since, all instances of conditional mean of $Y$ are independent,
$\sum_W Q_B(w_i)\bar{Y}\leq \sum_W Q_B(w_i)\{\hat{\bar{Y}}+O(C_4^1/4(\eta+\epsilon)^3/4\sigma)\}$ with probability at least, $(1-1/poly(n))^{\vert W\vert}$

$L(h_R,h_R^\infty\vert T=t)\leq \sum_W (\hat{\bar{Y}}\vert W=w_i,T=t)(Q_B(w_i)-\hat{Q}_B(w_i))-\sum_W Q(W_i)O(C_4^1/4(\eta+\epsilon)^3/4\sigma)$

$L(h_R,h^\infty_R)\leq \vert h(x_R)-h^\infty_R(x_R)\vert T=1+\vert h(x_R),h^\infty(x_R)\vert T=0)\vert$\\
$L(h_R,h^\infty_R)\leq \vert\alpha_{T=1}(\hat{Q_B},\hat{\beta_B})- \alpha_{T=1}(Q_B,\beta_R)-\gamma\vert+\vert\alpha_{T=0}(\hat{Q_B},\hat{\beta_B})-\alpha_{T=0}(Q_B,\beta_B)-\gamma\vert$ 
\end{proof}

\section{Proof of Corollary 5.2.1}
\begin{corollary}
For IPW estimator, if $P(W)=Q(W)=R(W), L(h^\infty(x_R),h^\infty(x_P))=2\gamma$
\end{corollary}
\begin{proof}
For the sake of proof, we'll assume the loss function is L1. 
Since,
\begin{align*}
    L(h_R^\infty,h_P^\infty\vert T=1)&=L((E_R[Y\vert T=1]-E_P[Y\vert T=1])\\
    &=\vert E_R\left[Y\vert T=1\right] -E_P\left[Y\vert T=0\right]\\
    &\leq \sum_w \vert R(W\vert T=1)E_R\left[Y\vert T=1,W\right]-P(W\vert T=1)E_P\left[Y\vert T=1,W\right]\vert\\
    &= \sum_w \vert\beta_{ipw}(Q_B,W) Q_B(W\vert T=1)E_{Q_B}\left[Y\vert T=1,W\right]-P(W\vert T=1)E_P\left[Y\vert T=1,W\right]\vert\\
    &\leq \sum_w \vert Q_B(W)\left(E_P\left[Y\vert T=1,W\right]+\gamma\right)-P(W\vert T=1)E_P\left[Y\vert T=1,W\right]\vert\\
    &= \sum_w \vert\gamma Q_B(W)+\big(Q_B(W)-P(W)\big)E_P\left[Y\vert T=1,W\right]\vert\\
    &=\sum_w\vert\gamma Q_B(W)\vert\\
    &=\gamma
\end{align*}
A similar argument can be made for $T=0$ and hence, $L(h^\infty(x_R),L^\infty(x_P))=2\gamma$
\end{proof}

\section{Results of Robust IPW}

\paragraph{Setup.} We generate data using the following structural equations (assuming both W and U are unidimensional), and present results for the following set of parameters:
\begin{align}
    n=10000,\alpha=0.0,\beta=0.01,\nu=0.3, \gamma=1.0,\delta=10.0\\
    noise_y\sim Normal(0,0.5)
    w \sim Binomial(p=0.7)  \\
    u \sim Normal(\mu=5.0,\sigma=1.0)       \\
    t = Bernoulli(p=sigmoid(\alpha\cdot w+\beta\cdot u)))      \\
    y = \nu\cdot w + \gamma t+\delta \cdot u + noise_y \label{y-eqn}
\end{align}
We simulate the Huber contamination due to $U$ as follows: with probability $\eta$, $\delta=\delta$, and with probability $1-\eta$, $\delta=0$.

Since $\gamma=1$, the true ACE (Average Causal Effect) is 1.0. The following table shows the robust IPW and standard IPW estimates over 10 different runs. 

\begin{tabular} { |l | c | c | }
 \hline
$\eta$ & Robust IPW (min,max) & Standard IPW(min,max) \\
 \hline
 0.0  & (0.979,1.011)  &  (0.979,1.011)  \\
 0.05  & (0.884,0.914)  & (0.860,1.377)  \\
 0.1  & (0.793,0.831)  & (0.820,1.737)  \\
 0.15  & (0.716,0.753)  & (0.709,1.887)  \\
 0.20  & (0.646,0.692)  & (0.227,1.640)  \\
\hline
\end{tabular}

\iffalse
\begin{tabular} { |l | c | c | }
 \hline
$\eta$ & Robust IPW (mean,var) & Vanilla IPW(mean,var) \\
 \hline
 0.0  & (0.800,0.0321)  &  (0.800,0.0321)  \\
 0.05  & (0.805, 0.0090)  & (1.076, 0.0400)  \\
 0.1  & (0.649, 0.0211)  & (1.090, 0.083)  \\
 0.15  & (0.584, 0.017)  & (1.018, 0.136)  \\
 0.20  & (0.661, 0.0001)  &  (1.073,0.161)  \\
\hline
\end{tabular}
\fi